\documentclass{uai2021} % for initial submission
\usepackage[american]{babel}
\usepackage{natbib} % has a nice set of citation styles and commands
\usepackage{mathtools} % amsmath with fixes and additions
\usepackage{booktabs} % commands to create good-looking tables
\usepackage{tikz} % nice language for creating drawings and diagrams

\usepackage{amssymb}
\usepackage{amsthm}
\usepackage{graphicx}
\usepackage[ruled,vlined]{algorithm2e}
\newcommand\Myperm[2][^n]{^{#1\mkern-2.5mu}{}P_{#2}}

\newtheorem{theorem}{Theorem}
\newtheorem{corollary}[theorem]{Corollary}
\newtheorem{definition}[theorem]{Definition}
\newtheorem{Proposition}[theorem]{Proposition}
\newtheorem{lemma}[theorem]{Lemma}
\graphicspath{ {./images/} }

\title{Multiclass Classification using dilute bandit feedback}

\author[1]{\href{mailto:Gaurav Batra <gaurav.batra@students.iiit.ac.in>}{Gaurav Batra}{}}
\author[2]{\href{mailto:Naresh Manwani <naresh.manwani@iiit.ac.in>}{Naresh Manwani}{}}
\affil[1]{%
    Machine Learning Lab\\
    IIIT Hyderabad\\
    gaurav.batra@students.iiit.ac.in
}

\affil[2]{%
    Machine Learning Lab\\
    IIIT Hyderabad\\
    naresh.manwani@iiit.ac.in
    
}

\begin{document}
\maketitle

\begin{abstract}
  This paper introduces a new online learning framework for multiclass classification called \textit{learning with diluted bandit feedback}. At every time step, the algorithm predicts a candidate label set instead of a single label for the observed example. It then receives a feedback from the environment whether the actual label lies in this candidate label set or not. This feedback is called "diluted bandit feedback". Learning in this setting is even more challenging than the bandit feedback setting \cite{Kakade2008}, as there is more uncertainty in the supervision. We propose an algorithm for multiclass classification using dilute bandit feedback (MC-DBF), which uses the exploration-exploitation strategy to predict the candidate set in each trial. We show that the proposed algorithm achieves $\mathcal{O}(T^{1-\frac{1}{m+2}})$ mistake bound if candidate label set size (in each step) is $m$. We demonstrate the effectiveness of the proposed approach with extensive simulations.
\end{abstract}

\section{Introduction}

In multi-class classification, the learning algorithm is given access to the examples and their actual class labels. The goal is to learn a classifier which given an example, correctly predicts its class label. This is called the full information setting. In the full information setting, online algorithms for multiclass classification are discussed in \cite{4066017,inproceedings,DBLP:conf/sdm/MatsushimaSYNN10}. In many applications, we do not get labels for all the examples. Instead, we can only access whether the predicted label for an example is correct. This is called bandit feedback setting \cite{Kakade2008}. Bandit feedback-based learning is useful in several web-based applications, such as sponsored advertising on web pages and recommender systems as mentioned in \cite{Kakade2008}.

In the linearly separable case, Kakade et. al \cite{Kakade2008} propose Banditron algorithm which can learn using bandit feedbacks. Banditron makes $\mathcal{O}(\sqrt{T})$ expected number of mistakes in the linearly separable case and $\mathcal{O}(T^{2/3})$ in the worst case. On the other hand Newtron \cite{NIPS2011_fde9264c} (based on the online Newton method) achieves $\mathcal{O}(\log T)$ regret bound in the best case and $\mathcal{O}(T^{2/3})$ regret in the worst case. Beygelzimer et. al \cite{DBLP:journals/corr/BeygelzimerOZ17} propose Second Order Banditron (SOBA) which achieves $\mathcal{O}(\sqrt{T})$ regret in the worst case.

\begin{figure*}[t]
\centerline{\fbox{\includegraphics[width = \textwidth]{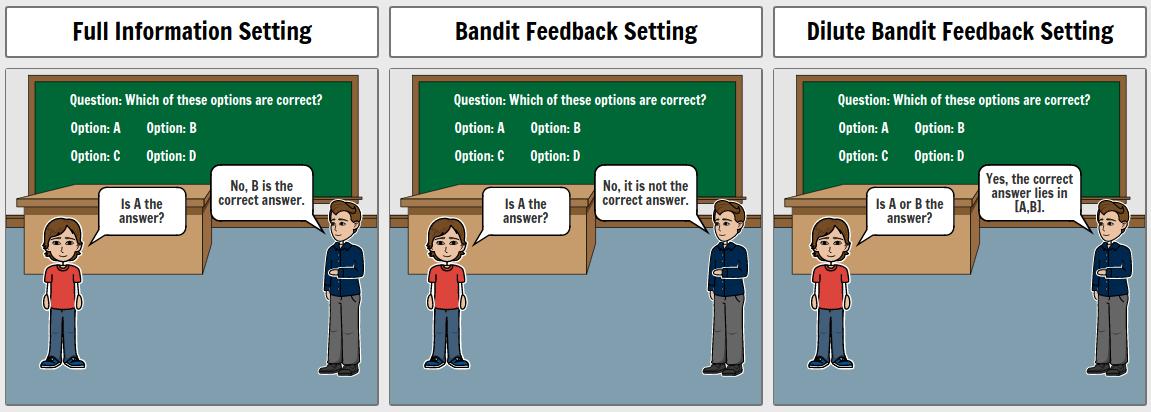}}}
\caption{The three types of supervised learning settings are explained in this figure. (a) Full Information Setting: In this setting, the agent receives the correct label on prediction. (b) Bandit Feedback Setting: Here, the agent gets the information whether his prediction is accurate or not. (c) Partial Bandit Feedback Setting: In this setting, the agent predicts a set of labels and gets the feedback whether the correct label lies in the predicted set or not.}
\label{fig:Setting}
\end{figure*}

In the bandit feedback-based approaches, the algorithm predicts a single label and seeks the feedback whether the predicted label is correct. Here, we introduce a new learning framework called "learning under diluted bandit feedback". At every time step, when the algorithm observes a new example, it predicts a candidate label set instead of a single label. Now the algorithm seeks the oracle's feedback whether the actual label lies in this candidate label set or not. Note that learning in this setting is even more challenging than bandit feedback setting as there is another level of uncertainty in the supervision. That is, if the feedback says that the actual label lies in the predicted candidate label set, we still do not know which of the label in the candidate set is the true one. Using the example presented in Figure~\ref{fig:Setting}, we can see the difference between bandit feedback and diluted bandit feedback.

Diluted bandit feedback-based learning can be useful in many applications. For example, consider the situation where a doctor is trying to diagnose a patient. Based on the patient's initial symptoms, she starts the treatment with some idea about possible diseases. Based on the treatment outcome, the doctor would know whether the actual disease was correctly diagnosed in the possible diseases guessed by the doctor. The result of the treatment here is diluted bandit feedback. It does not tell the exact disease but only indicates whether the actual disease lies in a possible set of diseases.

Another example would be that of advertising on web pages. The user first queries the system. Based on the query and user-specific information, the system makes a prediction as a set of advertisements. Finally, the user may either click on one of the ads or just ignore all of them. The action of the user clicking or ignoring the advertisements is the dilute bandit feedback.
In the cases mentioned above, knowing the ground truth label beforehand may not always be possible. Hence, this provides the motivation for coming up with the setting and the corresponding algorithm.

Note that diluted bandit feedbacks make the supervision weaker than bandit feedbacks. In this paper, we attempt the problem of learning multiclass classifier using diluted bandit feedbacks. To the best of our knowledge, this is the first work in this direction. Following are the key contributions in this paper:
\begin{itemize}
    \item We propose an algorithm which learns multiclass classifier using diluted bandit feedbacks. 
    \item We show that the proposed algorithm achieves sub-linear mistake bound of  $\mathcal{O}(T^{1-(m+2)^{-1}})$, where $m$ is the size of the subset predicted.
    \item We experimentally show that the proposed approach learns efficient classifiers using diluted bandit feedbacks.
\end{itemize}

The main novelty of the MC-DBF algorithm is that it is able to train even under the dilute bandit feedback setting. Comparing the dilute bandit feedback with the bandit and full information setting, we see that the amount of feedback that MC-DBF receives at the time of training is very less. In spite of this, our algorithm achieves an error rate that is comparable to that of algorithms that receive bandit feedback (Banditron) or full feedback (Perceptron) at the time of training. The dilute bandit feedback setting has been introduced for the first time in our paper and this algorithm is the first approach to train a classifier in this type of setting.

\section{Problem Setting: Diluted Bandit Feedback}

We now formally describe the problem statement for our multi-class classification with diluted bandit feedback. The classification is done in a sequence of rounds. At each round $t$, the algorithm observes an instance $\mathbf{x}^t \in \mathbb{R}^d$. The algorithm predicts a set of labels $\Tilde{Y}^t \subset \{1,\ldots,k\}$ such that $|\Tilde{Y}^t |=m$.
After predicting the set of labels, we observe the feedback $\mathbb{I}{\{y^t \in \Tilde{Y}^t\}}$ where $y^t$ is the true label corresponding to the input $\mathbf{x}^t$. $\mathbb{I}{\{y^t \in \Tilde{Y}^t\}}$ is $1$ if $y^t \in \Tilde{Y}^t$ else $0$. In this type of bandit feedback, the classifier receives the information if the predicted set contains the correct label or not. There are two possibilities of the value of $m$.
\begin{itemize}
    \item Case 1 ($m=1$): Here, $\Tilde{Y}^t=\Tilde{y}^t$ and the feedback reduces to $\mathbb{I}{\{y^t = \Tilde{y}^t\}}$ which is discussed in \cite{Kakade2008,10.5555/2986459.2986559,pmlr-v129-arora20a}. Thus, when $\mathbb{I}{\{y^t = \Tilde{y}^t\}}=1$, we know the true label. On the other hand, $\mathbb{I}{\{y^t = \Tilde{y}^t\}}=0$, the true label can be anything among $[k]\setminus \Tilde{y}_t$.
    \item Case 2 ($1<m<k$): Here, the uncertainty is present in both possibilities of the feedback $\mathbb{I}{\{y^t \in \Tilde{Y}^t\}}$. When $\mathbb{I}{\{y^t \in \Tilde{Y}^t\}}=1$, then the true label could be anything among the labels in the set $\Tilde{Y}^t$. When $\mathbb{I}{\{y^t \in \Tilde{Y}^t\}}=0$, the true label lies in the set $[k]\setminus \Tilde{Y}^t$. Thus, in both possibilities of the feedback, there is ambiguity about the true label. Hence the name {\bf diluted bandit feedback}.
\end{itemize}
This paper is mainly concerned about Case 2, where $1<m<k$ (diluted bandit feedback setting). The algorithm's final goal is to minimize the number of prediction mistake $\Hat{M}$ as defined below. 
\begin{equation}
        \begin{split}
            \Hat{M} := \sum_{t=1}^{T} \mathbb{I}{\{y^t \notin \Hat{Y}^t\}}
        \end{split}
\end{equation}
To the best of our knowledge, this is first time diluted bandit feedback setting has been discussed.

\section{Proposed Approach}
The algorithm tries to learn a linear classifier parameterized by a weight matrix $W \in \mathbb{R}^{k\times d}$. %Let $\mathbf{w}_1,\ldots, \mathbf{w}_k$ be the rows of matrix $W$.
To formulate the algorithm which learns using diluted bandit feedback, let us first look at a simple full information approach.

\subsection{Multiclass Algorithm with Subset Label Prediction: A Full Information Approach} Consider the approach where the algorithm can predict a subset of labels. We first define {\em label set prediction function}. 
\begin{definition}
{\bf Label Set Prediction Function} $\hat{Y}(\mathbf{x},W)$:  Given an example $\mathbf{x}$ and a weight matrix $W\in \mathbb{R}^{k\times d}$, we denote predicted label set of size $m$ as $\hat{Y}(\mathbf{x},W)$. We define
 $\hat{Y}(\mathbf{x},W):= \{a_1,\dots,a_m\}$
where 
$$a_i = \underset{j \in [k]\setminus\{a_1,\dots,a_{i-1}\}}{\arg\max}\;(W\mathbf{x})_j.$$
\end{definition}

\begin{algorithm}[t]
\SetAlgoLined
 Parameters: $\gamma \in$ (0,1.0)\;
 Initialize $W^{1}$ = 0 $\in \mathbb{R}^{k\times d}$ \;
 \For{t= 1,\dots,T}
 {
    Receive $\mathbf{x}^{t} \in R^{d}$\;
    %Set $\hat{Y}^t = \arg\max_{r \in [k]}\; (W^{t}\mathbf{x}^{t})_r$\;
    Predict $\Hat{Y}(\mathbf{x}^t,W^t)$ and receive feedback $y^t$\;
    Define $U^{t}_{r,j} = x^{t}_{j}\Big(\mathbb{I}{\{r = y^t\}} - \frac{1}{m}\mathbb{I}\{r \in \Hat{Y}(\mathbf{x}^t,W^t)\}\Big)$\;
    Update: $W^{t+1} = W^{t}+\Tilde{U}^{t}$\;
 }
 \caption{MC-SLP: Multiclass Classification using Subset Label Prediction}
 \label{alg:MCSLP}
\end{algorithm}
Thus, $\Hat{Y}(\mathbf{x},W)$ predicts top $m$-labels based on $m$-largest values in the vector $W\mathbf{x}$. Then we observe the true label $y\in [k]$.\footnote{Note that this setting is exactly opposite to the partial label setting \cite{Bhattacharjee2020,Arora2021}. In the partial label setting, ground truth is a labelled subset, and the algorithm predicts a single label.} We use following variant of 0-1 loss to capture the discrepancy between the true label ($y$) and the predicted label set $\Hat{Y}(\mathbf{x}, W)$.
\begin{equation}
\label{eq:partial-0-1-loss}
    \begin{split}
        L_{0-1} = \mathbb{I}{\{y \notin \hat{Y}(\mathbf{x},W)\}}
    \end{split}
\end{equation}
But, this loss is not continuous. So, we use following average hinge loss as a surrogate loss function,
\begin{equation}
    % \begin{split}
        L_{avg}(W,(\mathbf{x},y)) = [1-(W\mathbf{x})_{y}+\frac{1}{m}\sum_{i \in \hat{Y}(\mathbf{x},W)} (W\mathbf{x})_i]_{+}
    % \end{split}
\end{equation}
where $[A]_{+} = A$ if $A > 0$ else $0$. It is easy to see that $L_{avg}$ is upper bound to $L_{0-1}$.
\begin{lemma}
$L_{avg}$ is an upper-bound on $\mathbb{I}{\{y \in \Hat{Y}(\mathbf{x},W)\}}$, that is $L_{avg} \geq \mathbb{I}{\{y \notin \Hat{Y}(\mathbf{x},W)\}}$.
\end{lemma}
An online algorithm in this setting can be easily derived by using stochastic gradient descent on the loss $L_{avg}$. We call it MC-SLP (multiclass classification using subset label prediction). The algorithm works as follows. At trial $t$ we observe example $\mathbf{x}^t$. We predict the label set $\hat{Y}(\mathbf{x}^t,W^t)$ using the existing parameters $W^t$. Then we observe the true label $y^t$. We update the parameters using stochastic gradient descent on the loss $L_{avg}$ which results in the update equation $W^{t+1}=W^{t} + U^t$ where $U^{t}$ is described as follows.
\begin{equation}
\label{eq:update1}
    \begin{split}
        U^{t}_{r,j} = x^{t}_{j} \Big(\mathbb{I}{\{r = y_{t}\}} - \frac{\mathbb{I}{\{r \in \hat{Y}(\mathbf{x}^t,W^t)\}}}{m}\Big).
    \end{split}
\end{equation}
We repeat this process for $T$ number of trials.
The complete description of this approach is given in Algorithm~\ref{alg:MCSLP}. Note that MC-SLP is a full information type algorithm as we get access to the true label for each example. The following is true for MC-SLP (Algorithm~\ref{alg:MCSLP}).
\begin{lemma}
Let $W^t$ be the weight matrix in the beginning of trial $t$ of MC-SLP and $U^t$ be the update matrix in trial $t$ by MC-SLP. Let $\langle W^t,U^t \rangle=\sum_{r=1}^k\sum_{j=1}^d W^t_{r,j}U^t_{r,j}$ (matrix inner product). Then,
\begin{equation*}
    \begin{split}
        L_{avg}(W^t,(\mathbf{x}^t,y^t)) \geq \mathbb{I}{\{y^t \notin \Hat{Y}(\mathbf{x}^t,W^t)\}} \;-\; \langle W^t,U^t \rangle
    \end{split}
\end{equation*}
 \end{lemma}
This lemma gives us a lower bound on the loss $L_{avg}$ computed for example $\mathbf{x}^t$ at trial $t$. This is a useful result which we will need later.
\subsection{MC-DBF: Multiclass Learning with Diluted Bandit Feedback}

\begin{algorithm}[t]
\SetAlgoLined
 Parameters: $\gamma \in$ (0,1.0)\;
 Initialize $W^1= \mathbf{0}^{d\times k}$\;
 \For{$t= 1,\dots,T$}
 {
    Receive $\mathbf{x}^{t} \in \mathbb{R}^{d}$\;
    Find $\Hat{Y}(\mathbf{x}^t,W^t)$\;
    %= \{a_1,a_2,\dots,a_m\}$
    %where 
    %$a_i = \underset{j \in [k]\setminus\{a_1,\dots,a_{i-1}\}}{\arg\max}\;(W^t\mathbf{x}^t)_j$\;
    %Set $\hat{Y}^t = {\arg\max}_{r \in [k]}\; \mathbf{w}^t_r\cdot \mathbf{x}^t$\;
    Define $\mathbb{P}(r) := \frac{(1-\gamma)}{m}\mathbb{I}{\{r \in \hat{Y}(\mathbf{x}^t,W^t)\}} + \frac{\gamma}{k},\; \forall r \in [k]$\;
    Define $Z(A) = \mathbb{P}(b_1)\mathbb{P}(b_2|b_1)\dots \mathbb{P}(b_m|b_1,\dots,b_{m-1}),\;\forall A = \{b_1,\dots,b_m\}\in \mathbb{S}$\;
    Randomly sample $\Tilde{Y}^t$ according to $Z$\; 
    Predict $\Tilde{Y}^t$ and receive feedback $\mathbb{I}{\{y^t \in \Tilde{Y}^t\}}$\;
    Compute $\forall r\in [K]$ and $\forall j\in[d]$\; $\Tilde{U}^{t}_{r,j} = x^{t}_{j}\left[\frac{\mathbb{I}{\{y^t \in \Tilde{Y}^t\}}\mathbb{I}{\{r \in \Tilde{Y}^t\}}}{Z(\Tilde{Y}^t)\tau_1} - \tau_2 - \frac{\mathbb{I}\{r \in \Hat{Y}(\mathbf{x}^t,W^t)\}}{m}\right]$\;
    Update: $W^{t+1} = W^{t}+\Tilde{U}^{t}$\;
 }
 \caption{MC-DBF: Multiclass Classification Using Diluted Bandit Feedback}
 \label{alg:MC-DBF}
\end{algorithm}
We now describe the algorithm for learning using diluted bandit feedback. Here, for each example $\mathbf{x}^t$, we do not receive the true label $y^t$. We instead receive the feedback whether $y^t$ lies in the predicted label set $\Tilde{Y}^t$ (i.e. $\mathbb{I}\{y^t \in \Tilde{Y}^t\}$). The algorithm works as follows. 

At each iteration $t$, we receive $\mathbf{x}^t$ as input. We find $\Hat{Y}(\mathbf{x}^t,W^t)= \{a_1,a_2,\dots,a_m\}$
where 
$$a_i = \underset{j \in [k]\setminus\{a_1,\dots,a_{i-1}\}}{\arg\max}\;(W^t\mathbf{x}^t)_j.$$
We define probability distribution $\mathbb{P}$ on individual class labels as follows.
\begin{equation}
\label{eq:arm-prob}
    \begin{split}
        \mathbb{P}(r) = \frac{(1-\gamma)}{m} \mathbb{I}{\{r \in \Hat{Y}(\mathbf{x}^t,W^t)\}} + \frac{\gamma}{k},\; \forall r \in [k]
    \end{split}
\end{equation}
Here, $\gamma$ is the exploration parameter.
Let $\mathbb{S}$ denote the set of all $m$ size subsets of $\{1,\ldots,k\}$. We call them {\bf superarms} of size $m$. Now, we define probability distribution $Z$ on the set $\mathbb{S}$ as follows. For all $A=\{b_1,\dots,b_m\}\in \mathbb{S}$, we define
\begin{equation*}
        Z(A) =  \mathbb{P}(b_1)\mathbb{P}(b_2|b_1)\dots \mathbb{P}(b_m|b_1,\dots,b_{m-1}),
\end{equation*}
where $\mathbb{P}(b_i | b_{1},\dots,b_{i-1}) = \frac{\mathbb{P}(b_i)}{(1-\mathbb{P}(b_{1})\;-\;\cdots\;-\;\mathbb{P}(b_{i-1}))}$. 
$Z(A)$ is the probability of choosing $b_1,b_2,\dots,b_m$ from the set $[k]$ without replacement.\footnote{We see that $\sum_{A}Z(A) = 1$ as follows.
\begin{align*}
        \sum_{A} Z(A) &= \sum_{A} \mathbb{P}(b_1)\dots \mathbb{P}(b_m|b_1,b_2,\dots,b_{m-1})\\
        {} &= \sum_{b_1}\mathbb{P}(b_1)\dots\sum_{b_m}\frac{\mathbb{P}(b_m)}{(1-\mathbb{P}(b_1)\dots-\mathbb{P}(b_{m-1}))}
\end{align*}
But,  $\sum_{b_i} \frac{\mathbb{P}(b_i)}{(1-\mathbb{P}(b_1)-\dots-\mathbb{P}(b_{i-1})} = 1$. Thus, $\sum_{A}Z(A) = 1$.}

We randomly sample $\Tilde{Y}^t$ from $Z$ and predict $\Tilde{Y}^t$. We receive the diluted bandit feedback $\mathbb{I}{\{y^t \in \Tilde{Y}^t\}}$. We update the weight matrix as $W^{t+1}=W^t+\Tilde{U}^t$, where
\begin{align}
\label{eq:update-matrix} 
     \Tilde{U}^{t}_{r,j}  = x^{t}_{j}\Big(\frac{\mathbb{I}{\{y^t \in \Tilde{Y}^t\}}\mathbb{I}{\{r \in \Tilde{Y}^t\}}}{Z(\Tilde{Y}^t)\tau_1}  - \frac{\mathbb{I}\{r \in \Hat{Y}(\mathbf{x}^t,W^t)\}}{m}
  - \tau_2\Big)
\end{align}
where $\tau_{1}= m\;\Myperm[k-2]{m-1}$ and 
$\tau_{2} = \frac{m-1}{k-m}$.
$\Tilde{U}^{t}$ accesses $y^t$ only through the indicator $\mathbb{I}{\{y^t \in \Tilde{Y}^t\}}$. We will see that $\Tilde{U}^t$ is an unbiased estimator of $U^t$ using the following Lemma.

\begin{lemma}
\label{lemma:unbiased-estimator}
Consider $\Tilde{U}^t$ defined in Eq.(\ref{eq:update-matrix}). Then, $\mathbb{E}_{Z}[\Tilde{U}^t] = U^t$, where $U^t$ is defined in Eq.(\ref{eq:update1}).
\end{lemma}

\section{Mistake Bound Analysis of MC-DBF}
In this section, we derive the expected mistake bound ($\mathbb{E}_{Z}[\Hat{M}]$) for the proposed approach MC-DBF (Algorithm~\ref{alg:MC-DBF}). 
To get the mistake bound, we first need to derive some intermediate results. We first derive an upper bound the expected value of the Frobenius norm of the update matrix $\Tilde{U}^t$.

\begin{lemma}
       \begin{align*}
    &\mathbb{E}_{Z}[||\Tilde{U}^t||^2_{F}] \leq    ||\mathbf{x}^t||^2_2\Big[\frac{mk^m\;\Myperm[k]{m}}{\gamma^m\tau_1^2 }-\frac{2m\tau_2}{\tau_1} -\frac{2}{\tau_1}\\
    & + {k}[\tau_2^2+\frac{1}{km}+\frac{2\tau_2}{k}]\Big]+\frac{2 ||\mathbf{x}^t||^2_2}{\tau_1}\mathbb{I}\{y^t\notin \Hat{Y}(\mathbf{x}^t,W^t)\}
\end{align*}

%where $\tau_1$ and $\tau_2$ are as defined in Eq.(\ref{eq:tau-def}).
\end{lemma}
Now we derive the expected mistake bound $\mathbb{E}_{Z}[\Hat{M}]$ using Theorem as follows.

\begin{theorem}
Assume that for the sequence of examples, $(\mathbf{x}^1,y^1),\dots,(\mathbf{x}^t,y^t)$, we have, for all $t$, $\mathbf{x}^t \in \mathbb{R}^d$, $||\mathbf{x}^t|| \leq 1$ and $y^t \in [k]$.
Let $W^*$ be any matrix and let $R_T$ be the cumulative average hinge loss of $W^*$ defined as follows.
\begin{equation*}
    \begin{split}
        R_T & = \sum_{t=1}^T L_{avg}(W^*,(\mathbf{x}^t,y^t)).
        %& = \sum_{t=1}^{T}[1-(W^*\mathbf{x}^t)_{y^t}+\frac{1}{m}\sum_{i \in \Hat{Y}^t} (W^*\mathbf{x}^t)_i]_{+}.
    \end{split}
\end{equation*}
Let $D$ be the complexity of $W^*$ defined as below:
$ D  = 2||W^*||^2_{F}$. Then the number of mistakes $\Hat{M}$ made by the Algorithm~2 satisfies
\begin{align}
        \nonumber \mathbb{E}_{Z}[\Hat{M}] & \leq R_{T} +  \sqrt{\frac{\lambda_1DR_{T}}{2}} + 3\max\left(\frac{\lambda_1D}{2},\sqrt{\frac{(\lambda_2+1)D T}{2}}\right)\\
        \label{eq:mistake-bound}&\qquad + \gamma T.
\end{align}
where $\Hat{M} = \sum_{t=1}^{T} \mathbb{I}{\{y^t \notin \Hat{Y}(\mathbf{x}^t,W^t)\}}$, $\lambda_1 = \frac{2}{\tau_1}$ and 
\begin{equation*}
    \begin{split}
        \lambda_2 &= \big[\frac{mk^m\;\Myperm[k]{m}}{\gamma^m\tau_1^2 }-\frac{2m\tau_2}{\tau_1} -\frac{2}{\tau_1}+{k}[\tau_2^2+\frac{1}{km}+\frac{2\tau_2}{k}]\big]
    \end{split}
\end{equation*}
\end{theorem}
We will now analyze different cases and find out the mistake bound in those cases. Before going ahead, we state a new separability definition as follows.
\begin{definition}
 {\bf Linear Separability:} A sequence of examples, $(\mathbf{x}^1,y^1),\dots,(\mathbf{x}^T,y^T)$ is linearly separable if there exists a matrix $W^* \in \mathbb{R}^{k\times d}$ such that
 $$(W^*\mathbf{x}^t)_{y^t}-(W^*\mathbf{x}^t)_i\;\geq \;1,\;\forall i\neq y^t,\;\forall t\in [T].$$
 \end{definition}
 Note that linear separability also implies $(W^*\mathbf{x}^t)_{y^t}-\frac{1}{m}\sum_{i\in \Hat{Y}(\mathbf{x},W^*)}(W^*\mathbf{x}^t)_i \geq 1,\forall t\in [T]$. Which implies $L_{avg}(W^*,(\mathbf{x}^t,y^t))=0,\;\forall t\in[T]$. Thus, when a sequence of examples $(\mathbf{x}^1,y^1),\dots,(\mathbf{x}^T,y^T)$ is linearly separable with respect to a weight matrix $W^*$, then 
 $$R_T=\sum_{t=1}^TL_{avg}(W^*,(\mathbf{x}^t,y^t))=0.$$
 \begin{corollary}
 Let $(\mathbf{x}^1,y^1),\dots,(\mathbf{x}^T,y^T)$ be the sequence of examples which are linearly separable. Then algorithm MC-DBF achieves $\mathcal{O}(T^{(1-\frac{1}{(m+2)})})$ mistake bound on it.
 \end{corollary}

\begin{corollary}
Moreover, if we consider $R_{T} \leq \mathcal{O}(T^{(1-\frac{1}{(m+2)})})$, by setting $\gamma = \mathcal{O}( (\frac{1}{T})^{\frac{1}{(m+2)}})$, we get that $\mathbb{E}_{Z}[\Hat{M}] \leq \mathcal{O}(T^{(1-\frac{1}{(m+2)})})$.
\end{corollary}

Thus, we see that on increasing $m$, running time complexity of the algorithm also increases.

% Moreover, we also observe that by setting $m = 1$, which corresponds to the bandit feedback setting discussed in \cite{Kakade2008} we get:
% \begin{align*}
%         \mathbb{E}_{Z}[\Hat{M}] \leq \mathcal{O}(\sqrt{T})
% \end{align*}
% This is expected as Banditron makes $\mathcal{O}(\sqrt{T})$ expected number of mistakes in the linearly separable case.

% \begin{corollary}
% The above result also holds if the value of hinge loss is small enough, that is if $R_{T} \leq \mathcal{O}(T^{(1-\frac{1}{2(m+1)})})$, setting $\gamma = (\frac{D}{T})^{\frac{1}{2(m+1)}}$, we get that $\mathbb{E}_{Z}[\Hat{M}] \leq \mathcal{O}(T^{(1-\frac{1}{2(m+1)})})$.
% \end{corollary}

\begin{figure*}[thb]
% \vspace{.3in}
\centerline{\fbox{\includegraphics[width = \textwidth,height = 34em]{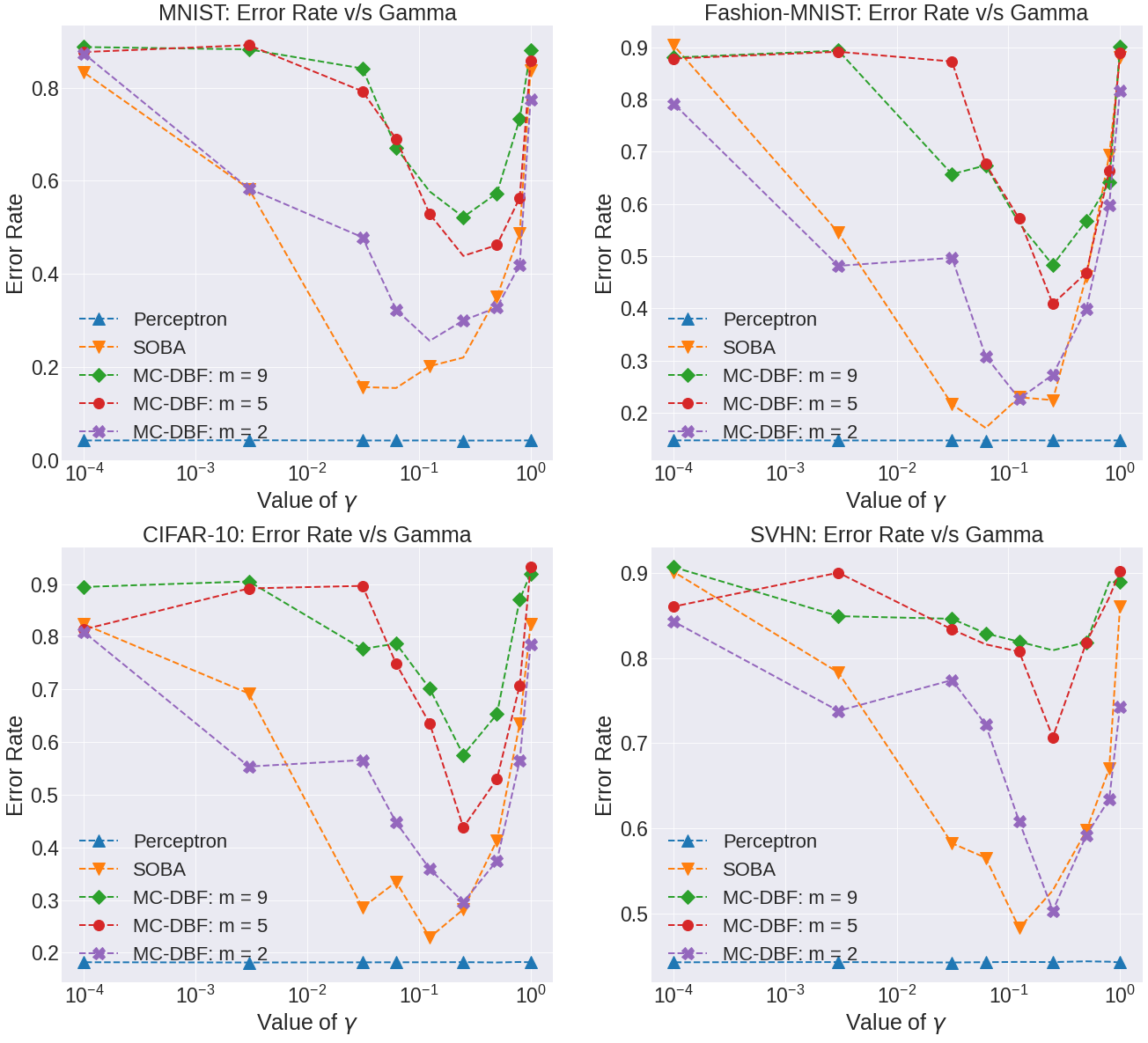}}}
\caption{Searching best value of $\gamma$: graphs show converged error rates for MC-DBF (with different values of $m$) and Second Order Banditron (SOBA) for varying values of $\gamma$. The $\gamma$ values on the X-axis are on a $\log$ scale.}
\label{fig:Optimal_Gamma}
\end{figure*}

\begin{figure*}[thb]
\centerline{\fbox{\includegraphics[width = \textwidth,height = 34em]{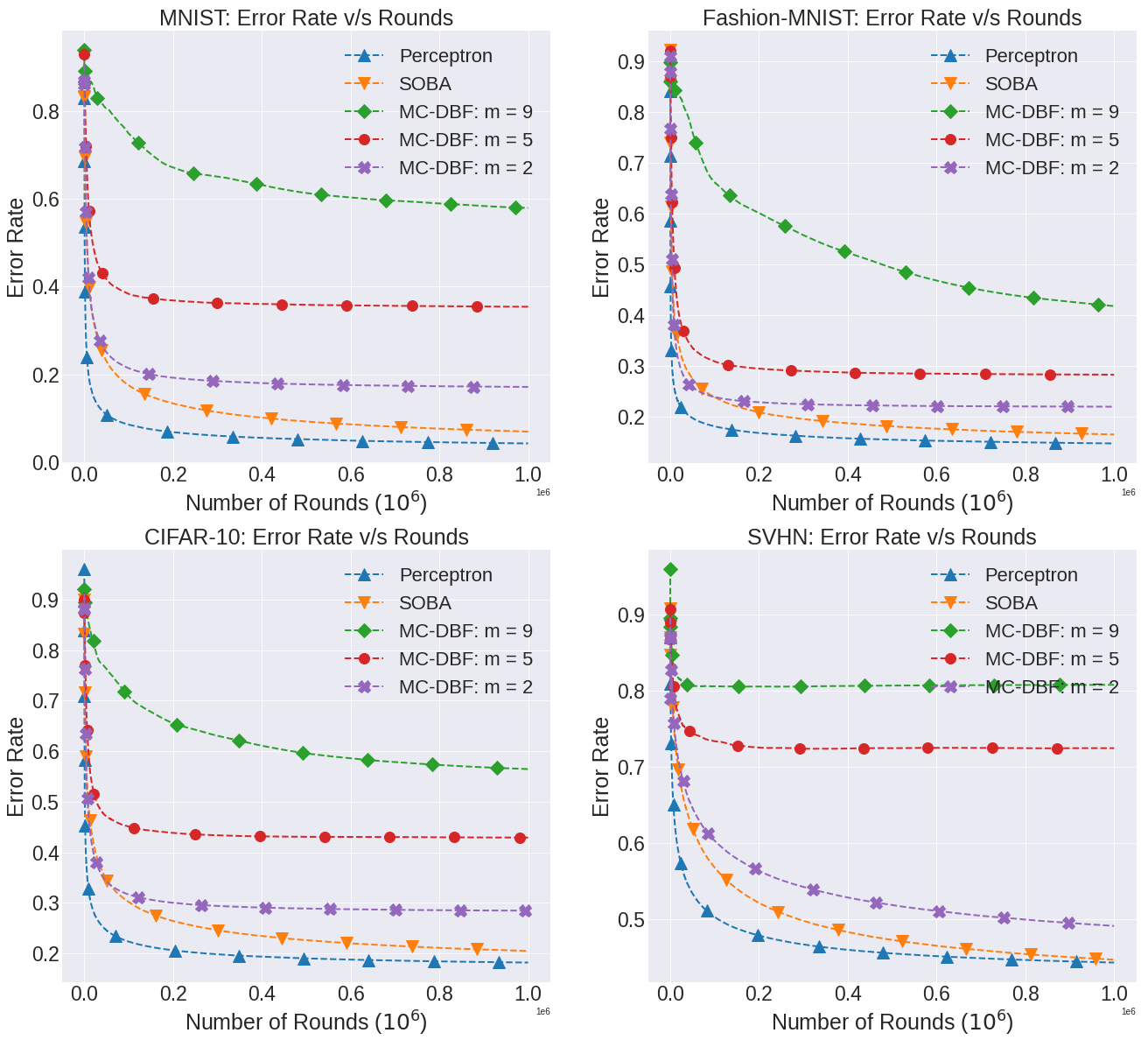}}}
\caption{Comparison of MC-DBF Algorithm using different values of $m$ with Perceptron (benchmark algorithm with full information feedback) and Second Order Banditron (SOBA) (benchmark algorithm for bandit feedback).}
\label{fig:Comparison}
\end{figure*}

\begin{figure*}[tb]
% \vspace{.3in}
\centerline{\fbox{\includegraphics[width = \textwidth,height = 34em]{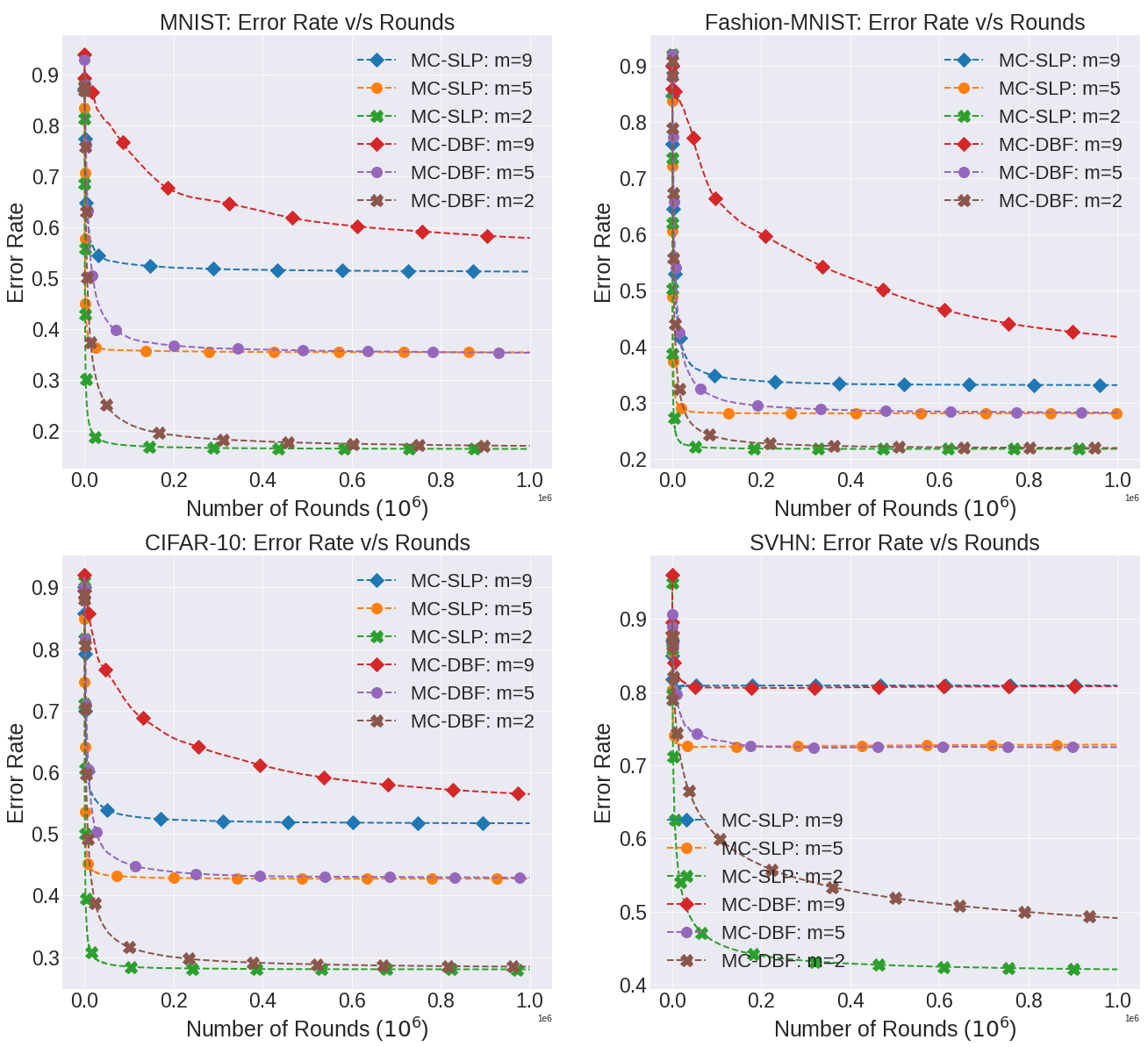}}}
% \vspace{.3in}
\caption{Comparison of MC-DBF Algorithm with different values of $m$ with MC-SLP}
\label{fig:Convergence}
\end{figure*}

\begin{figure}[tb]
\centerline{\fbox{\includegraphics[width = 20em, height = 14em]{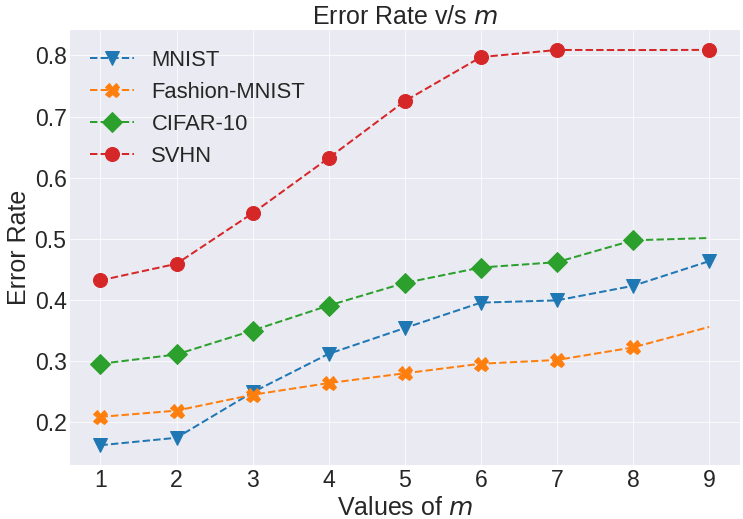}}}
\caption{Plot to show the variation of the error rate of MC-DBF algorithm with different values of $m$.}
\label{fig:Varying_m}
\end{figure}

\section{Experiments}
This section shows the experimental results of the proposed algorithm MC-DBF (Algorithm~\ref{alg:MC-DBF}) and its comparison with other benchmark algorithms on various datasets. 
%Then we will move on to compare the performance of our proposed algorithm with various algorithms already present that work in the full information setting like Perceptron, or the bandit feedback setting like the Banditron. Furthermore, we will also investigate the effect of varying $m$ and $\gamma$ on the MC-DBF.

\subsection{Datasets Used and Preprocessing}
We use CIFAR-10 \cite{krizhevsky2009learning}, SVHN \cite{SVHN}, MNIST
\cite{lecun2010mnist} and Fashion-MNIST \cite{xiao2017/online} datasets to show experimental results.
%\begin{itemize}
%    \item \textbf{CIFAR-10}: The CIFAR-10 dataset \cite{krizhevsky2009learning} consists of $60,000$ images of $32\times32\times3$ dimension, which belong to a total of $10$ classes. This means each class has $6,000$ images in the dataset.
%    \item \textbf{CIFAR-100}: The CIFAR-100 dataset \cite{krizhevsky2009learning} is very similar to the  CIFAR-10 dataset. It consists of $60,000$ images of $32\times32\times3$ dimension, which belong to a total of $100$ classes. This means each class has $600$ images in the dataset.
    %\item \textbf{MNIST}: The MNIST dataset \cite{lecun2010mnist} consists of $70,000$ images of handwritten digits of $28\times28\times1$ dimension, which belong to a total of $10$ classes ($0-9$). This means each digit/class has $7,000$ grayscale images in the dataset.
    %\item \textbf{Fashion-MNIST}: The Fashion-MNIST dataset \cite{xiao2017/online} is a dataset which is very similar to MNIST dataset mentioned above. This dataset uses grayscale images of clothing items instead of handwritten digits and is used as a direct drop-in replacement for the traditional MNIST dataset. It consists of $70,000$ images of $28\times28\times1$ dimension, which belong to a total of $10$ classes. This means each class has $7,000$ images in the dataset.
%\end{itemize}
We use VGG-16 \cite{DBLP:journals/corr/SimonyanZ14a} model that is pre-trained on ImageNet dataset to extract the features. We use TensorFlow framework \cite{tensorflow2015-whitepaper} on all the datasets mentioned above for feature extraction using VGG-16. All the experiments have been executed on a machine with Intel (R) Xeon(R) CPU @ 2.30 GHz with 12.72 Gb of RAM.

The images are passed through the VGG-16 network, and then relevant features are extracted from the last layer of VGG-16 network. The final dimension of the features extracted for each dataset is 512.

In addition to the datasets mentioned above, we also use a synthetic linearly separable dataset SYNSEP \cite{Kakade2008} for the purpose of comparison.
%\begin{itemize}
%    \item CIFAR-10 : $60,000\times512$, k = $10$
%    \item CIFAR-100 : $60,000\times512$, k = $100$
%    \item MNIST : $70,000\times512$, k = $10$
%    \item Fashion-MNIST : $70,000\times512$, k = $10$
%\end{itemize}

\subsection{Benchmark Algorithms}
We compare the proposed approach MC-DBF with Perceptron \cite{4066017}, which is an algorithm for full information setting, Second Order Banditron (SOBA) \cite{DBLP:journals/corr/BeygelzimerOZ17}, which is a bandit feedback algorithm. We also compare with MC-SLP, which is a full information version of MC-DBF. We will use three types of values of $m$ - (low, medium, high) in MC-DBF to observe the effect of variations in $m$. 

\subsection{Experimental Setup and Performance Metric}
We run each of the algorithms mentioned above for $1,000,000$ iterations for $10$ independent runs for every dataset. In each iteration, we calculate the error rate (number of incorrectly classified samples averaged over the total number of rounds). For calculating the error rate we compare the ground truth label $y_t$, with the predicted label $\Hat{y}^t=\arg\max_{j\in[k]}\; (W^t\mathbf{x}^t)_j$. The final plots have error rate averaged over the $10$ independent trials on the Y-axis and number of trials on the X-axis.

\subsection{Choosing Optimal value of $\gamma$ in MC-DBF}
MC-DBF takes $\gamma$ as a parameter. To choose the parameter's best values, we plot the trend of error rates for varying values of $\gamma$. We choose the value of $\gamma$ for which the error rate value is minimized from these plots. While calculating the error rate, we compare the true label $y_t$ with the predicted label $\Hat{y}^t=\arg\max_{j\in[k]}\; (W^t\mathbf{x}^t)_j$. We use a similar process to get optimal values of hyper-parameters for the other benchmark algorithms.

Figure~\ref{fig:Optimal_Gamma} shows the trend of the converged error rates with $\log(\gamma)$ on all the datasets for all the algorithms. The best values of $\gamma$ for MCDBF for different datasets have been summarized in Table~\ref{tab:banditron}. All the final plots have been made using these optimal values.
\begin{table}
\centering
\begin{tabular}{|c|c|}
\hline
\multicolumn{2}{|c|}{SOBA}\\
\hline
Dataset  & $\gamma$\\
\hline
MNIST       & 0.02 \\
Fashion-MNIST            & 0.1  \\
CIFAR-10    & 0.3  \\
SVHN   & 0.25  \\
\hline
\end{tabular}
\begin{tabular}{|c|c|}
\hline
\multicolumn{2}{|c|}{MC-DBF ($m$: low)}\\
\hline
Dataset  & $\gamma$\\
\hline
MNIST       & 0.12  \\
Fashion-MNIST            & 0.12  \\
CIFAR-10    & 0.28  \\
SVHN   & 0.31  \\
\hline
\end{tabular}
\begin{tabular}{|c|c|}
\hline
\multicolumn{2}{|c|}{MC-DBF ($m$: medium)}\\
\hline
Dataset  & $\gamma$\\
\hline
MNIST       & 0.12  \\
Fashion-MNIST            & 0.2  \\
CIFAR-10    & 0.25  \\
SVHN   & 0.3  \\
\hline
\end{tabular}
\begin{tabular}{|c|c|}
\hline
\multicolumn{2}{|c|}{MC-DBF ($m$: high)}\\
\hline
Dataset  & $\gamma$\\
\hline
MNIST       & 0.25  \\
Fashion-MNIST            & 0.3  \\
CIFAR-10    & 0.3  \\
SVHN   & 0.5  \\
\hline
\end{tabular}
\caption{Optimal values of $\gamma$ for different algorithms}
\label{tab:banditron}
\end{table}

\subsection{Comparison of MC-BDF with Benchmarking algorithms}
Figure \ref{fig:Comparison} presents the comparison results of our proposed algorithm (MC-DBF) using different values of $m$ with Perceptron \cite{4066017} (full information setting) and Second Order Banditron (SOBA) \cite{DBLP:journals/corr/BeygelzimerOZ17} (bandit feedback setting).\\
We observe the MC-DBF algorithm performs well for low and medium values of $m$. Its performance is comparable to the Banditron algorithm in almost all the datasets and even comparable to Perceptron in the Fashion-MNIST and SVHN datasets for low values of $m$. We must keep in mind that increasing the value of $m$ forces the value of feedback to become more diluted, which is correctly demonstrated by high values of $m$ in the figure.

Figure \ref{fig:Convergence} show the comparison reults of our proposed algorithm MC-DBF with MC-SLF (full information version of MC-DBF) for different values of $m$. We observe that MC-DBF converges to MC-SLF in all the datasets, which is expected as according to Lemma \ref{lemma:unbiased-estimator} $E_{Z}[\Tilde{U}] = U$.

Figure \ref{fig:Linearly_Separable} shows the comparison of the various benchmark algorithms with our proposed algorithm (MC-DBF) using different values of $m$ on the SYNSEP dataset, which is a linearly separable dataset. The $x$ and $y$ axes have been plotted on a log scale. From the figure, we observe that our proposed algorithm performs comparable to Second Order Banditron (SOBA), in the linearly separable case as well. 

\subsection{Effect of changing values of $m$}
Figure \ref{fig:Varying_m} shows the trend of the error rate of MC-DBF versus $m$ for all the datasets. We observe that the error rate increases on increasing $m$ which is not surprising as increasing $m$ implies feedback to the algorithm becomes increasing dilute leading to increasing error-rate for the same number of rounds and constant $\gamma$.

\begin{figure}[tb]
\centerline{\fbox{\includegraphics[width = 20em,height = 14em]{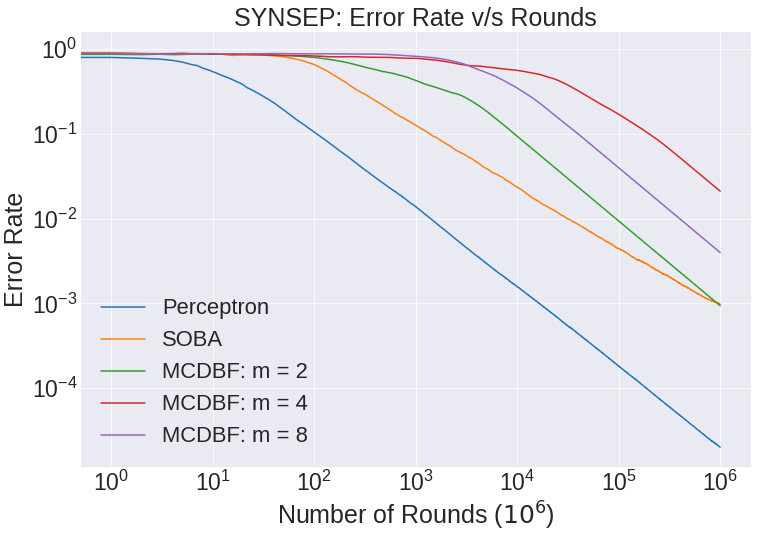}}}
\caption{Plot to show the variation of the error rate of MC-DBF algorithm with other benchmark algorithms on SYNSEP (linearly separable dataset)}
\label{fig:Linearly_Separable}
\end{figure}

\section{Conclusion}
This paper proposed a multiclass classification algorithm that uses diluted bandit feedback for training, namely MC-BDF. We used the exploration-exploitation strategy to predict a subset of labels in each trial. We then update the matrix, using an unbiased estimator of the MC-SLP update matrix (full information version of MC-DBF). We also proved the upper bound for the expected number of mistakes made by our algorithm. We also experimentally compared MC-DBF with other benchmark algorithms for the full/bandit feedback settings on various datasets. The results show that our algorithm MC-DBF performs comparably to the benchmark algorithms, despite receiving lesser feedback on most of the datasets.

\bibliography{uai2021-template}

\begin{thebibliography}{16}
\providecommand{\natexlab}[1]{#1}
\providecommand{\url}[1]{\texttt{#1}}
\expandafter\ifx\csname urlstyle\endcsname\relax
  \providecommand{\doi}[1]{doi: #1}\else
  \providecommand{\doi}{doi: \begingroup \urlstyle{rm}\Url}\fi

\bibitem[Abadi et~al.(2015)Abadi, Agarwal, Barham, Brevdo, Chen, Citro,
  Corrado, Davis, Dean, Devin, Ghemawat, Goodfellow, Harp, Irving, Isard, Jia,
  Jozefowicz, Kaiser, Kudlur, Levenberg, Man\'{e}, Monga, Moore, Murray, Olah,
  Schuster, Shlens, Steiner, Sutskever, Talwar, Tucker, Vanhoucke, Vasudevan,
  Vi\'{e}gas, Vinyals, Warden, Wattenberg, Wicke, Yu, and
  Zheng]{tensorflow2015-whitepaper}
Mart\'{\i}n Abadi, Ashish Agarwal, Paul Barham, Eugene Brevdo, Zhifeng Chen,
  Craig Citro, Greg~S. Corrado, Andy Davis, Jeffrey Dean, Matthieu Devin,
  Sanjay Ghemawat, Ian Goodfellow, Andrew Harp, Geoffrey Irving, Michael Isard,
  Yangqing Jia, Rafal Jozefowicz, Lukasz Kaiser, Manjunath Kudlur, Josh
  Levenberg, Dan Man\'{e}, Rajat Monga, Sherry Moore, Derek Murray, Chris Olah,
  Mike Schuster, Jonathon Shlens, Benoit Steiner, Ilya Sutskever, Kunal Talwar,
  Paul Tucker, Vincent Vanhoucke, Vijay Vasudevan, Fernanda Vi\'{e}gas, Oriol
  Vinyals, Pete Warden, Martin Wattenberg, Martin Wicke, Yuan Yu, and Xiaoqiang
  Zheng.
\newblock {TensorFlow}: Large-scale machine learning on heterogeneous systems,
  2015.
\newblock URL \url{http://tensorflow.org/}.
\newblock Software available from tensorflow.org.

\bibitem[Arora and Manwani(2020)]{pmlr-v129-arora20a}
Maanik Arora and Naresh Manwani.
\newblock Exact passive-aggressive algorithms for multiclass classification
  using bandit feedbacks.
\newblock In \emph{Proceedings of The 12th Asian Conference on Machine
  Learning}, volume 129, pages 369--384, Bangkok, Thailand, 18--20 Nov 2020.

\bibitem[Arora and Manwani(2021)]{Arora2021}
Maanik Arora and Naresh Manwani.
\newblock Exact passive aggressive algorithm for multiclass classification
  using partial labels.
\newblock In \emph{8th ACM IKDD CODS and 26th COMAD}, page 38–46, 2021.

\bibitem[Beygelzimer et~al.(2017)Beygelzimer, Orabona, and
  Zhang]{DBLP:journals/corr/BeygelzimerOZ17}
Alina Beygelzimer, Francesco Orabona, and Chicheng Zhang.
\newblock Efficient online bandit multiclass learning with
  {\textdollar}{\textbackslash}tilde\{O\}({\textbackslash}sqrt\{T\}){\textdollar}
  regret.
\newblock \emph{CoRR}, abs/1702.07958, 2017.
\newblock URL \url{http://arxiv.org/abs/1702.07958}.

\bibitem[Bhattacharjee and Manwani(2020)]{Bhattacharjee2020}
Rajarshi Bhattacharjee and Naresh Manwani.
\newblock Online algorithms for multiclass classification using partial labels.
\newblock In \emph{Proceedings of the 24th Pacific-Asia Conference on Knowledge
  Discovery and Data Mining (PAKDD)}, pages 249--260, 2020.

\bibitem[Crammer and Singer(2003)]{4066017}
Koby Crammer and Yoram Singer.
\newblock Ultraconservative online algorithms for multiclass problems.
\newblock \emph{J. Mach. Learn. Res.}, 3\penalty0 (null):\penalty0 951–991,
  March 2003.

\bibitem[Fink et~al.(2006)Fink, Shalev-Shwartz, Singer, and
  Ullman]{inproceedings}
Michael Fink, Shai Shalev-Shwartz, Yoram Singer, and Shimon Ullman.
\newblock Online multiclass learning by interclass hypothesis sharing.
\newblock pages 313--320, 01 2006.
\newblock \doi{10.1145/1143844.1143884}.

\bibitem[Hazan and Kale(2011{\natexlab{a}})]{10.5555/2986459.2986559}
Elad Hazan and Satyen Kale.
\newblock Newtron: An efficient bandit algorithm for online multiclass
  prediction.
\newblock In \emph{Proceedings of the 24th International Conference on Neural
  Information Processing Systems}, page 891–899, 2011{\natexlab{a}}.

\bibitem[Hazan and Kale(2011{\natexlab{b}})]{NIPS2011_fde9264c}
Elad Hazan and Satyen Kale.
\newblock Newtron: an efficient bandit algorithm for online multiclass
  prediction.
\newblock In J.~Shawe-Taylor, R.~Zemel, P.~Bartlett, F.~Pereira, and K.~Q.
  Weinberger, editors, \emph{Advances in Neural Information Processing
  Systems}, volume~24, pages 891--899. Curran Associates, Inc.,
  2011{\natexlab{b}}.
\newblock URL
  \url{https://proceedings.neurips.cc/paper/2011/file/fde9264cf376fffe2ee4ddf4a988880d-Paper.pdf}.

\bibitem[Kakade et~al.(2008)Kakade, Shalev-Shwartz, and Tewari]{Kakade2008}
Sham~M. Kakade, Shai Shalev-Shwartz, and Ambuj Tewari.
\newblock Efficient bandit algorithms for online multiclass prediction.
\newblock In \emph{Proceedings of the 25th International Conference on Machine
  Learning}, ICML '08, page 440–447, 2008.

\bibitem[Krizhevsky et~al.(2009)Krizhevsky, Hinton,
  et~al.]{krizhevsky2009learning}
Alex Krizhevsky, Geoffrey Hinton, et~al.
\newblock Learning multiple layers of features from tiny images.
\newblock 2009.

\bibitem[LeCun et~al.(2010)LeCun, Cortes, and Burges]{lecun2010mnist}
Yann LeCun, Corinna Cortes, and CJ~Burges.
\newblock Mnist handwritten digit database.
\newblock \emph{ATT Labs [Online]. Available:
  http://yann.lecun.com/exdb/mnist}, 2, 2010.

\bibitem[Matsushima et~al.(2010)Matsushima, Shimizu, Yoshida, Ninomiya, and
  Nakagawa]{DBLP:conf/sdm/MatsushimaSYNN10}
Shin Matsushima, Nobuyuki Shimizu, Kazuhiro Yoshida, Takashi Ninomiya, and
  Hiroshi Nakagawa.
\newblock Exact passive-aggressive algorithm for multiclass classification
  using support class.
\newblock In \emph{Proceedings of the {SIAM} International Conference on Data
  Mining, {SDM} 2010, Columbus, Ohio, {USA}}, pages 303--314, 2010.

\bibitem[Netzer et~al.(2011)Netzer, Wang, Coates, Bissacco, Wu, and Ng]{SVHN}
Yuval Netzer, Tao Wang, Adam Coates, Alessandro Bissacco, Bo~Wu, and Andrew Ng.
\newblock Reading digits in natural images with unsupervised feature learning.
\newblock \emph{NIPS}, 01 2011.

\bibitem[Simonyan and Zisserman(2014)]{DBLP:journals/corr/SimonyanZ14a}
Karen Simonyan and Andrew Zisserman.
\newblock Very deep convolutional networks for large-scale image recognition.
\newblock \emph{CoRR}, abs/1409.1556, 2014.
\newblock URL \url{http://arxiv.org/abs/1409.1556}.

\bibitem[Xiao et~al.(2017)Xiao, Rasul, and Vollgraf]{xiao2017/online}
Han Xiao, Kashif Rasul, and Roland Vollgraf.
\newblock Fashion-mnist: a novel image dataset for benchmarking machine
  learning algorithms, 2017.

\end{thebibliography}

\newpage
\title{Multiclass Classification using dilute bandit feedback - Supplementary Material}
\maketitle

\section{Proof of Lemma~2}
\label{lem: Lavg_lower_bound_loose}
There are basically two cases that we have to analyse.

\begin{itemize}
    \item $y^t \notin \Hat{Y}(\mathbf{x}^t,W^t)$: In this case, we see that $\frac{1}{m} \sum_{i \in \Hat{Y}(\mathbf{x}^t,W^t)} (W^t\cdot \mathbf{x}^t)_{i} \geq (W^t\cdot \mathbf{x}^t)_{y_t}$ due to the way $\Hat{Y}(\mathbf{x}^t,W^t)$ is generated. Thus, 
\begin{equation*}
        \begin{split}
            L_{avg} = 1-(W^t\cdot \mathbf{x}^t)_{y_t}+\frac{1}{m}\sum_{i \in \Hat{Y}(\mathbf{x}^t,W^t)} (W^t\cdot \mathbf{x}^t)_{i} \geq 1 \\
            = \mathbb{I}{\{y^t \notin \Hat{Y}(\mathbf{x}^t,W^t)\}}
        \end{split}
\end{equation*}
\item $y^t \in \Hat{Y}(\mathbf{x}^t,W^t)$:
\begin{equation*}
        \begin{split}
            L_{avg} = \Big[1-(W^t\cdot \mathbf{x}^t)_{y_t}+\frac{1}{m}\sum_{i \in \Hat{Y}(\mathbf{x}^t,W^t)} (W^t\cdot \mathbf{x}^t)_{i}\Big]_{+}\\
            \geq 0 = \mathbb{I}{\{y^t \notin \Hat{Y}(\mathbf{x}^t,W^t)\}}
        \end{split}
\end{equation*}
\end{itemize}
Therefore, $L_{avg} \geq \mathbb{I}{\{y^t \notin \hat{Y}(\mathbf{x}^t,W^t)\}}$.

\section{Proof of Lemma~3}
\label{lem: Lavg_lower_bound}
We know from the definition of $U^t$, that

\begin{equation}
    \begin{split}
        U^t_{r,j}  = x^{t}_{j} \Big(\mathbb{I}{\{r = y_{t}\}} - \frac{\mathbb{I}{\{r \in     \hat{Y}(\mathbf{x}^t,W^t)\}}}{m}\Big)
    \end{split}
    \label{Update-Matrix}
\end{equation}
Substituting this in $\langle W^t,U^t \rangle$ results in,
\begin{equation}
    \begin{split}
        \langle W^t,U^t \rangle \ = (W^t\cdot \mathbf{x}^t)_{y^t} - \frac{1}{m} \sum_{i \in \Hat{Y}(\mathbf{x}^t,W^t)} (W^t\cdot \mathbf{x}^t)_{i}
    \label{eq:Forbenium Norm of update matrix}
    \end{split}
\end{equation}

We now have three cases:

\begin{itemize}
    \item $y^t \notin \Hat{Y}(\mathbf{x}^t,W^t)$: This implies that 
    \begin{equation*}
        \begin{split}
            (W^t\cdot \mathbf{x}^t)_{i} \geq (W^t\cdot \mathbf{x}^t)_{y^t} \quad \forall i \in \Hat{Y}(\mathbf{x}^t,W^t)  
        \end{split}
    \end{equation*}
    Summing over all $i\in \Hat{Y}(\mathbf{x}^t,W^t)$, we get
    \begin{align*}
            &\frac{1}{m}\sum_{i\in \Hat{Y}(\mathbf{x}^t,W^t)}(W^t\cdot \mathbf{x}^t)_{i} \geq (W^t\cdot \mathbf{x}^t)_{y_t}  \\ 
            \Rightarrow\;\; & 1-(W^t\cdot \mathbf{x}^t)_{y_t}+\frac{1}{m}\sum_{i\in \Hat{Y}(\mathbf{x}^t,W^t)}(W^t\cdot \mathbf{x}^t)_{i} \geq
            0.
            \end{align*}
    Thus,
    \begin{equation*}
        \begin{split}
            L_{avg}(W^t,(\mathbf{x}^t,y^t)) & = 1-(W^t\cdot \mathbf{x}^t)_{y_t}+\frac{1}{m}\sum_{i \in \Hat{Y}(\mathbf{x}^t,W^t)} (W^t\cdot \mathbf{x}^t)_{i}\\ 
            & = 1 \;- \;\langle W^t,U^t \rangle\\
            & = \mathbb{I}\{y^t \notin \hat{Y}(\mathbf{x}^t,W^t)\} \;- \;\langle W^t,U^t \rangle
        \end{split}
    \end{equation*}
    
    \item $y^t \in \Hat{Y}(\mathbf{x}^t,W^t)$ and $\frac{1}{m}\sum_{i}(W^t\cdot \mathbf{x}^t)_{i} \geq (W^t\cdot \mathbf{x}^t)_{y^t}$:
    In this case also, 
    \begin{equation*}
        \begin{split}
            L_{avg}(W^t,(\mathbf{x}^t,y^t)) & = 1-(W^t\cdot \mathbf{x}^t)_{y_t}+\frac{1}{m}\sum_{i \in \Hat{Y}(\mathbf{x}^t,W^t)} (W^t\cdot \mathbf{x}^t)_{i} \\
            & = 1 - \langle W^t,U^t \rangle \\ 
            {} & \geq - \langle W^t,U^t \rangle\\
            & = \mathbb{I}\{y^t \notin \hat{Y}(\mathbf{x}^t,W^t)\} \;- \;\langle W^t,U^t \rangle
        \end{split}
    \end{equation*}
    
    \item $y^t \in \Hat{Y}(\mathbf{x}^t,W^t)$ and $\frac{1}{m}\sum_{i}(W^t\cdot \mathbf{x}^t)_{i} < (W^t\cdot \mathbf{x}^t)_{y^t}$:

    In this case, using Eq.(\ref{eq:Forbenium Norm of update matrix}), we get 
    \begin{equation*}
        \begin{split}
            \langle W^t,U^t \rangle \ > 0
        \end{split}
    \end{equation*}
    As $L_{avg}(W^t,(\mathbf{x}^t,y^t)) \geq 0$ by definition of hinge loss, we get
    \begin{equation*}
        \begin{split}
            L_{avg}(W^t,(\mathbf{x}^t,y^t)) & \geq 0\\
            {} & > -\;\langle W^t,U^t \rangle\\
            & = \mathbb{I}\{y^t \notin \hat{Y}(\mathbf{x}^t,W^t)\} \;- \;\langle W^t,U^t \rangle
        \end{split}
    \end{equation*}
\end{itemize}
Hence, using from these three cases, we infer that
\begin{equation*}
    \begin{split}
        L_{avg}(W^t,(\mathbf{x}^t,y^t)) \geq \mathbb{I}{\{y^t \notin \Hat{Y}(\mathbf{x}^t,W^t)\}} - \langle W^t,U^t \rangle
    \end{split}
\end{equation*}

\section{Proof of Lemma~4}
\label{lem: Unbiased_Estimator}
For all $r \in [k]$ and for all $j \in [d]$ we have:

\begin{align}
       \nonumber  &\mathbb{E}_{Z}[\Tilde{U}^t_{r,j}]  = \sum_{A}Z(A) x^t_{j}\Big(\frac{\mathbb{I}{\{y^t \in A\}}\mathbb{I}{\{r \in A\}}}{Z(A)(m\Myperm[k-2]{m-1})} -\frac{\mathbb{I}{\{r \in \Hat{Y}(\mathbf{x}^t,W^t)\}}}{m}\\
        \nonumber & - \frac{m-1}{k-m}\Big)\\
        \label{eq:expectation-estimator}& = x^{t}_{j}\Big(\frac{\sum_{A} \mathbb{I}{\{y^t \in A\}}\mathbb{I}{\{r \in A\}}}{m\Myperm[k-2]{m-1}} - \frac{\mathbb{I}{\{r \in \Hat{Y}(\mathbf{x}^t,W^t)\}}}{m} - \frac{m-1}{k-m}\Big)
\end{align}

Now, we simplify the term $\sum_{A}\mathbb{I}{\{y^t \in A\}}\mathbb{I}{\{r \in A\}}$. We know that $\mathbb{I}{\{r = y^t\}} + \mathbb{I}{\{r \neq y^t\}} = 1$. We use this to simplify the above expression as follows:

\begin{align*}
        \sum_{A} \mathbb{I}{\{y^t \in A\}}\mathbb{I}{\{r \in A\}} &= \sum_{A} \mathbb{I}{\{y^t \in A\}}\mathbb{I}{\{r \in A\}}\mathbb{I}{\{r = y^t\}} \\
        {} & +\sum_{A} \mathbb{I}{\{y^t \in A\}}\mathbb{I}{\{r \in A\}}\mathbb{I}{\{r \neq y^t\}}
\end{align*}

Number of sets $A$ satisfying condition $\mathbb{I}{\{y^t \in A\}}\mathbb{I}{\{r \in A\}}\mathbb{I}{\{r = y^t\}}=1$ are $m(\Myperm[k-1]{m-1})$. Also, number of sets $A$ that satisfy the condition $\mathbb{I}{\{y^t \in A\}}\mathbb{I}{\{r \in A\}}\mathbb{I}{\{r \neq y^t\}}=1$ are $m(m-1)(\Myperm[k-2]{m-2})$. Therefore,

\begin{align}
     \nonumber    &\sum_{A}\mathbb{I}{\{y^t \in A\}}\mathbb{I}{\{r \in A\}} = m(\Myperm[k-1]{m-1})\mathbb{I}{\{r = y^t\}} + \\
     \nonumber    &\qquad \qquad m(m-1)(\Myperm[k-2]{m-2})\mathbb{I}{\{r \neq y^t\}}\\
     \label{eq:Sum-Temp}   {} & = m(\Myperm[k-2]{m-1})\mathbb{I}{\{r = y^t\}} +  m(m-1)(\Myperm[k-2]{m-2})
\end{align}

We can verify the expression by setting $m=1$ in Eq.(\ref{eq:Sum-Temp}). The L.H.S. evaluates to $\mathbb{I}{\{r = y^t\}}$ as $A$ reduces to a singleton set. Similarly, the R.H.S. also reduces to $\mathbb{I}{\{r = y^t\}}$.
Plugging the value obtained in Eq.(\ref{eq:Sum-Temp}) in Eq.(\ref{eq:expectation-estimator}), we get
\begin{align*}
        &\mathbb{E}_{Z}[\Tilde{U}_{r,j}^{t}] = x^{t}_{j}\Big(\mathbb{I}{\{r = y^t\}} + \frac{m(m-1)\;\;\Myperm[k-2]{m-2}}{m\;\;\Myperm[k-2]{m-1}} - \\
        &\qquad \qquad \;\;\;\; -\frac{1}{m}\mathbb{I}{\{r \in \Hat{Y}(\mathbf{x}^t,W^t)\}} -\frac{m-1}{k-m}\Big)\\
        {} &= x^{t}_{j}\Big(\mathbb{I}{\{r = y^t\}} - \frac{1}{m}\mathbb{I}{\{r \in \Hat{Y}(\mathbf{x}^t,W^t)\}}\Big)= U^t_{r,j} \ (Eq.(\ref{Update-Matrix}))
\end{align*}
Thus, we have proved that $\mathbb{E}_{Z}[\Tilde{U}^{t}] = U^{t}$.

\section{Proof of Lemma~5}
\label{lem: Expectation_of_update}
We first find a lower bound on the probability of a superarm $A$. 
\begin{Proposition}
\label{lemma1}
Let $A = \{b_1,\cdots,b_m\} \in \mathbb{S}$ and $Z(A)$ be the probability of selecting super-arm $A$.
%\begin{equation*}
 %   \begin{split}
  %      Z(A) = \mathbb{P}(b_1)\mathbb{P}(b_2|b_1)\dots \mathbb{P}(b_m|b_1,\dots,b_{m-1})
  %  \end{split}
%\end{equation*}
Then, we can show that
$Z(A) \geq \frac{\gamma}{k}\left({\frac{2}{m}}\right)^{m-1} e^{-(m-1)}.$
\end{Proposition}

\begin{proof}
As $Z(A)$ is the probability of choosing $\{b_1,b_2,\cdots,b_m\}$ from the set $[k]$ without replacement. We see that

\begin{equation*}
    \begin{split}
        Z(A) &= \mathbb{P}(b_1)\mathbb{P}(b_2|b_1)\dots \mathbb{P}(b_m|b_1,\dots,b_{m-1})\\
        {} &= \frac{\prod_{i}^{m}\mathbb{P}(b_{i})}{\prod_{i}^{m} (1 - \sum_{j=1}^{i-1} \mathbb{P}(b_{j}))}\geq \prod_{i}^{m}\mathbb{P}(b_{i})
    \geq \left(\frac{\gamma}{k}\right)^m,
    \end{split}
\end{equation*}

where the last inequality holds because $\mathbb{P}(b_{i})\geq \frac{\gamma}{k},\;\forall i$ (as $P(b_i) = \frac{(1-\gamma)}{m}\mathbb{I}{\{b_i \in \Hat{Y}(\mathbf{x}^t,W^t)\}} + \frac{\gamma}{k}$). 
\end{proof}

We first compute the value of $||\Tilde{U}^t||^2_{F}$. There are 2 cases possible.

\begin{itemize}
    \item $y^t \notin \Tilde{Y}^t$: $
        \Tilde{U}^t_{r,j} = x^t_j \Big[-\tau_2-\frac{\mathbb{I}{\{r \in \Hat{Y}(\mathbf{x}^t,W^t)\}}}{m}\Big]$
\begin{align*}
        ||\Tilde{U}^t||^2_{F} &= ||\mathbf{x}^t||^2_2\Big[m\big[\tau_2+\frac{1}{m}\big]^2+(k-m)\big[\tau_2\big]^2\Big]\\
        \mathbb{E}[||\Tilde{U}^t||^2_{F}|y^t \notin \Tilde{Y}^t] &=||\mathbf{x}^t||^2_2\Big[m\big[\tau_2+\frac{1}{m}\big]^2+(k-m)\big[\tau_2\big]^2\Big]
        %& \leq ||\mathbf{x}^t||^2_2\Big[m\big[\tau_2+\frac{1}{m}\big]^2+(k)\big[\tau_2\big]^2\Big]
\end{align*}
\item $y^t \in \Tilde{Y}^t$:% and $|\Hat{Y}(\mathbf{x}^t,W^t) \bigcap \Tilde{Y}^t| = b$: Note that $0 \leq b \leq m$.
\begin{align*}
        &||\Tilde{U}^t||^2_{F}  %||\mathbf{x}^t||^2_2\Big[b\big[\frac{1}{Z(\Tilde{Y}^t)\tau_1} - \tau_2 -\frac{1}{m} \big]^2\\
        %& +(m-b)\big[\frac{1}{Z(\Tilde{Y}^t)\tau_1} - \tau_2\big]^2 \\
        %&+ (m-b)\big[\tau_2 + \frac{1}{m}\big]^2+(k+b-2m)\big[\tau_2\big]^2\Big]\\
        = ||\mathbf{x}^t||^2_2\Big[\frac{m}{Z(\Tilde{Y}^t)^2\tau_1^2}-\frac{2m\tau_2}{Z(\Tilde{Y}^t)\tau_1}+m[\tau_2+\frac{1}{m}]^2 %+k\big[\tau_2\big]^2
        \\ & + (k-m)\big[\tau_2^2\big] -\frac{2}{Z(\Tilde{Y}^t)\tau_1}\sum_r\mathbb{I}\{y^t\in \Tilde{Y}^t\}\mathbb{I}\{y^t\in \Hat{Y}(\mathbf{x}^t,W^t)\}\Big]\\
        & \leq ||\mathbf{x}^t||^2_2\Big[\frac{m}{Z(\Tilde{Y}^t)^2\tau_1^2}-\frac{2m\tau_2}{\tau_1}+m[\tau_2+\frac{1}{m}]^2 %+k\big[\tau_2\big]^2
        \\ & +(k-m)\big[\tau_2^2\big]-\frac{2}{\tau_1}\Big]\mathbb{I}\{y^t\in \Hat{Y}(\mathbf{x}^t,W^t)\}\\
        &
        +||\mathbf{x}^t||^2_2\Big[\frac{m}{Z(\Tilde{Y}^t)^2\tau_1^2}-\frac{2m\tau_2}{\tau_1}+m[\tau_2+\frac{1}{m}]^2 \\
        & +(k-m)\big[\tau_2^2\big] \Big]\mathbb{I}\{y^t\notin \Hat{Y}(\mathbf{x}^t,W^t)\}\\
        & = -\frac{2||\mathbf{x}^t||^2_2}{\tau_1}\mathbb{I}\{y^t\in \Hat{Y}(\mathbf{x}^t,W^t)\}\\
        &
        +||\mathbf{x}^t||^2_2\Big[\frac{m}{Z(\Tilde{Y}^t)^2\tau_1^2}-\frac{2m\tau_2}{\tau_1}+m[\tau_2+\frac{1}{m}]^2 \\
        & +(k-m)\big[\tau_2^2\big] \Big]\\
        & = \frac{2||\mathbf{x}^t||^2_2}{\tau_1}\mathbb{I}\{y^t\notin \Hat{Y}(\mathbf{x}^t,W^t)\}\\
        &
        +||\mathbf{x}^t||^2_2\Big[\frac{m}{Z(\Tilde{Y}^t)^2\tau_1^2}-\frac{2m\tau_2}{\tau_1}+m[\tau_2+\frac{1}{m}]^2 \\
        & +(k-m)\big[\tau_2^2\big] -\frac{2}{\tau_1}\Big]
        \end{align*}
        Taking expectation on both sides, we get the following.
        \begin{align*}
        &\mathbb{E}[||\Tilde{U}^t||^2_{F}|y^t \in \Tilde{Y}^t] \leq 
        ||\mathbf{x}^t||^2_2\Big[\frac{m}{\tau_1^2}\sum_{\Tilde{Y}^t:y^t\in\Tilde{Y}^t}\frac{1}{Z(\Tilde{Y}^t)} \\
        &+m[\tau_2+\frac{1}{m}]^2  +(k-m)\big[\tau_2^2 \big]-\frac{2m\tau_2}{\tau_1}-\frac{2}{\tau_1}\Big]\\
        &+\frac{2||\mathbf{x}^t||^2_2}{\tau_1}\mathbb{I}\{y^t\notin \Hat{Y}(\mathbf{x}^t,W^t)\}
        %& = ||\mathbf{x}^t||^2_2\Big[\frac{1}{\tau_1^2}\sum_{\Tilde{Y}^t:y^t\notin\Tilde{Y}^t}\frac{1}{Z(\Tilde{Y}^t)} - \frac{2m\tau_2}{\tau_1}\\
        %&+m[\tau_2+\frac{1}{m}]^2  +(k-m)\big[\tau_2^2 \big]\Big]\mathbb{I}\{y^t\notin \Hat{Y}(\mathbf{x}^t,W^t)\}
        \end{align*}
        We use the lower bound on $Z(A),\;A\in\mathbb{S}$ derived in Proposition~\ref{lemma1}. Also, using the fact that $\vert \{\Tilde{Y}^t|y^t\in\Tilde{Y}^t\}\vert \;\leq \;\Myperm[k]{m}$, we get 
\begin{align*}
 &\mathbb{E}[||\Tilde{U}^t||^2_{F}|y^t \in \Tilde{Y}^t]   
\leq 
        ||\mathbf{x}^t||^2_2\Big[\frac{mk^m\;\Myperm[k]{m}}{\gamma^m\tau_1^2 }\\
        &+m[\tau_2+\frac{1}{m}]^2  +(k-m)\big[\tau_2^2 \big]-\frac{2m\tau_2}{\tau_1}-\frac{2}{\tau_1} \Big]\\
        &+\frac{2||\mathbf{x}^t||^2_2}{\tau_1}\mathbb{I}\{y^t\notin \Hat{Y}(\mathbf{x}^t,W^t)\}
\end{align*}
\end{itemize}
Let $\Gamma =\mathbb{P}(y^t \in \Tilde{Y}^t)$.
Thus, $\mathbb{E}_{Z}[||\Tilde{U}^t||^2_{F}]$ can be upper bounded as follows.
\begin{align*}
        &\mathbb{E}_{Z}[||\Tilde{U}^t||^2_{F}]  =\mathbb{P}(y^t \in \Tilde{Y}^t)\mathbb{E}[||\Tilde{U}^t||^2_{F}|y^t \in \Tilde{Y}^t]\\
        &  \quad + \left(1-\mathbb{P}(y^t \in \Tilde{Y}^t)\right)\mathbb{E}[||\Tilde{U}^t||^2_{F}|y^t \notin \Tilde{Y}^t]\\
        &\leq \Gamma ||\mathbf{x}^t||^2_2\Big[\frac{mk^m\;\Myperm[k]{m}}{\gamma^m\tau_1^2 }-\frac{2m\tau_2}{\tau_1} -\frac{2}{\tau_1}\\
        & \qquad +m[\tau_2+\frac{1}{m}]^2 + (k-m)\big[\tau_2^2\big] \Big]\\
        &+\frac{2\Gamma ||\mathbf{x}^t||^2_2}{\tau_1}\mathbb{I}\{y^t\notin \Hat{Y}(\mathbf{x}^t,W^t)\}\\
        &  + (1-\Gamma)||\mathbf{x}^t||^2_2\Big[m\big[\tau_2+\frac{1}{m}\big]^2
        +(k-m)\big[\tau_2\big]^2\Big]\\
        & = \Gamma ||\mathbf{x}^t||^2_2\Big[\frac{mk^m\;\Myperm[k]{m}}{\gamma^m\tau_1^2 }-\frac{2m\tau_2}{\tau_1} -\frac{2}{\tau_1}\Big]\\
        &+ ||\mathbf{x}^t||^2_2(k-m)\big[\tau_2\big]^2+||\mathbf{x}^t||^2_2m\big[\tau_2+\frac{1}{m}\big]^2\\
        &+\frac{2\Gamma ||\mathbf{x}^t||^2_2}{\tau_1}\mathbb{I}\{y^t\notin \Hat{Y}(\mathbf{x}^t,W^t)\}\\
        & \leq  ||\mathbf{x}^t||^2_2\Big[\frac{mk^m\;\Myperm[k]{m}}{\gamma^m\tau_1^2}-\frac{2m\tau_2}{\tau_1} -\frac{2}{\tau_1}\Big]\\
        &\qquad+||\mathbf{x}^t||^2_2(k-m)\big[\tau_2\big]^2+||\mathbf{x}^t||^2_2m\big[\tau_2+\frac{1}{m}\big]^2\\
        &+\frac{2 ||\mathbf{x}^t||^2_2}{\tau_1}\mathbb{I}\{y^t\notin \Hat{Y}(\mathbf{x}^t,W^t)\}\\
        \end{align*}
        Thus we get,
        \begin{align*}
            \mathbb{E}_{Z}[||\Tilde{U}^t||^2_{F}]
         & \leq  ||\mathbf{x}^t||^2_2\Big[\frac{mk^m\;\Myperm[k]{m}}{\gamma^m\tau_1^2 }-\frac{2m\tau_2}{\tau_1} -\frac{2}{\tau_1}\\
         &\;\; + {k}[\tau_2^2+\frac{1}{km}+\frac{2\tau_2}{k}]\Big]\\
         &+\frac{2 ||\mathbf{x}^t||^2_2}{\tau_1}\mathbb{I}\{y^t\notin \Hat{Y}(\mathbf{x}^t,W^t)\}
\end{align*}

\section{Proof of Theorem~6}
We start the proof by defining the inner product between two matrix $W^*$ and $W_t$ as follows:
\begin{equation}
    \begin{split}
       \langle W^*,W^t \rangle = \sum_{r=1}^k\sum_{j=1}^d W^*_{r,j}W^t_{r,j}
    \end{split}
\end{equation}
We will next try to upper-bound and lower-bound $\mathbb{E}_{Z}[\langle W^*,W^{T+1} \rangle]$, with the assumption that $W^1 = 0$. Let us first start with lower-bound.
\paragraph{Lower-bound of $\mathbb{E}_{Z}[\langle W^*,W^{T+1} \rangle]$:}
We first define a new variable $\Delta_{t}$:
\begin{equation}
    \begin{split}
       \Delta_t = \mathbb{E}_{Z}[\langle W^*,W^{T+1} \rangle] - \mathbb{E}_{Z}[\langle W^*,W^{t} \rangle]
     \end{split}
\end{equation}
Also, we know from the definition of Update Matrix $\Tilde{U}^t$, that $W^{t+1} = W^{t} + \Tilde{U}^t$.
Substituting this in the above equation gives, 
\begin{equation*}
    \begin{split}
        \Delta_t = \mathbb{E}_{Z}[\langle W^*,\Tilde{U}^t \rangle]
    \end{split}
\end{equation*}
Also, using Lemma 4, we obtain that $\forall t$,
\begin{equation*}
    \begin{split}
        \Delta_t = \mathbb{E}_{Z}[\langle W^*,{U}^t \rangle]
    \end{split}
\end{equation*}
Using Lemma~3, we know that,
\begin{equation}
    \begin{split}
        L_{avg}(W^*,(\mathbf{x}^t,y^t)) \geq \mathbb{I}{\{y^t \notin \Hat{Y}(\mathbf{x}^t,W^*)\}} - \langle W^*,{U}^t \rangle
    \end{split}
\end{equation}
Therefore, taking Expectation on both sides and summing over $t$ we get
\begin{equation}
    \begin{split}
        \sum_t \Delta_t & \geq \sum_t \mathbb{E}_{Z}[\mathbb{I}{\{y^t \notin \Hat{Y}(\mathbf{x}^t,W^t)\}}] - \sum_t L_{avg}(W^*,(\mathbf{x}^t,y^t))\\
        {} & \geq \mathbb{E}_{Z}[{M}] - R_T
    \end{split}
\end{equation}

where ${M} = \sum_t \mathbb{I}{\{y^t \notin \Hat{Y}(\mathbf{x}^t,W^t)\}}$ and $R_T$ is the cumulative hinge loss defined.

Therefore,
\begin{equation}
    \begin{split}
        \mathbb{E}_{Z}[\langle W^*,W^{T+1} \rangle] & = \sum_t \Delta_t \geq \mathbb{E}_{Z}[M] - R_{T}
    \end{split}
\end{equation}
Hence we get the lower-bound for $\mathbb{E}_{Z}[\langle W^*,W^{T+1} \rangle]$.
Now we will try to derive its upper bound.

\paragraph{Upper-bound on $\mathbb{E}_{Z}[\langle W^*,W^{T+1} \rangle]$:}
Using Cauchy-Schwartz inequality, we get the following,
\begin{equation}
    \begin{split}
        \langle W^*,W^{T+1} \rangle & \ \leq ||W^*||_{F}||W^{T+1}||_{F} 
    \end{split}
\end{equation}
where $||.||_{F}$ is the Frobenius Norm. Using the definition of $D$,the concavity of the square root function and Jensen's inequality, we can obtain that

\begin{equation}
    \begin{split}
        \mathbb{E}_{Z}[\langle W^*,W^{T+1} \rangle] & \leq \sqrt{\frac{D\mathbb{E}_{Z}[||W^{T+1}||_{F}^2]}{2}}
    \end{split}
    \label{eq: upper_bound_on_dot_weight}
\end{equation}
Now, we need to upper bound the expected value of $||W^{T+1}||_{F}^2$. Using the definition of $W^{T+1}$, we get

\begin{equation}
    \begin{split}
        \mathbb{E}_{Z}[||W^{T+1}||_{F}^2] &= \mathbb{E}_{Z}[||W^{T}||_{F}^2+\langle W^T,\Tilde{U}^T\rangle+||\Tilde{U}^T||_{F}^2]\\
        {} &= \sum_{t=1}^{T}\left(\mathbb{E}_{Z}[\langle W^t,\Tilde{U}^t\rangle]+\mathbb{E}_{Z}[||\Tilde{U}^t||_{F}^2]\right)
    \end{split}
    \label{eq: Expectation_of_forbenius_weight}
\end{equation}
Using Lemma 4 , we know that
\begin{equation}
    \begin{split}
        \mathbb{E}_{Z}[\langle W^t,\Tilde{U}^t\rangle] & = \mathbb{E}_{Z}[\langle W^t,U^t \rangle]
    \end{split}
\end{equation}
Using Lemma 5, we can also obtain an Upper bound on $\mathbb{E}_{Z}[||\Tilde{U}^t||_{F}^2]$ of the form:
\begin{equation}
    \begin{split}
        \mathbb{E}_{Z}[||\Tilde{U}^t||_{F}^2] & \leq 
        ||\mathbf{x}^t||^2\left[\lambda_{1}\mathbb{I}{\{y^t \notin \Hat{Y}(\mathbf{x}^t,W^t)\}} + \lambda_{2}\right]
        %||\mathbf{x}^t||^2_2\Bigg[m\big[\tau_2+\frac{1}{m}\big]^2+k\big[\tau_2\big]^2+ m[\frac{k^{2m}}{\gamma^{2m}\tau_1^2}+\tau_2^2]\Bigg]
    \end{split}
    \label{eq: Upper_bound_Expectation_of_forbenius_update}
\end{equation}
where $\lambda_1$ and $\lambda_2$ are as defined below:
\begin{equation}
    \begin{split}
        \lambda_1 &= \frac{2}{\tau_1}\\
        \lambda_2 &= \big[\frac{mk^m\;\Myperm[k]{m}}{\gamma^m\tau_1^2 }-\frac{2m\tau_2}{\tau_1} -\frac{2}{\tau_1}+{k}[\tau_2^2+\frac{1}{km}+\frac{2\tau_2}{k}]\big]
    \end{split}
\end{equation}
Also, using the Cauchy-Schwartz inequality again we get
\begin{equation}
    \begin{split}
        \langle W^t,U^t \rangle \ \leq ||W^t||_{F}^2||U^t||_{F}^2 \leq 1
    \end{split}
\end{equation}
If we use Eq. (\ref{eq: Expectation_of_forbenius_weight}) and (\ref{eq: Upper_bound_Expectation_of_forbenius_update}) and assume $||\mathbf{x}^t|| \leq 1 \ \forall t$, we can obtain that
\begin{equation}
    \begin{split}
        \mathbb{E}_{Z}[||W^{T+1}||^2] & \leq \sum_{t=1}^{T}\mathbb{E}_{Z}\left[\lambda_{1}\mathbb{I}{\{y^t \notin \Hat{Y}(\mathbf{x}^t,W^t)\}} + \lambda_{2} + 1\right]\\
        {} & \leq \lambda_1\mathbb{E}_{Z}[{M}] + (\lambda_2+1)T
    \end{split}
\end{equation}
Substituting this in Eq. (\ref{eq: upper_bound_on_dot_weight}) and using the inequality $\sqrt{a+b} \leq \sqrt{a}+\sqrt{b}$, we can upper-bound $\mathbb{E}_{Z}[\langle W^*,W^{T+1} \rangle]$ as follows:
\begin{equation}
    \begin{split}
        \mathbb{E}_{Z}[\langle W^*,W^{T+1} \rangle] & \leq \sqrt{\frac{D\mathbb{E}_{Z}[||W^{T+1}||_{F}^2]}{2}}\\
        {} & \leq \sqrt{\frac{\lambda_1D \mathbb{E}_{Z}[{M}]}{2}} + \sqrt{\frac{(\lambda_2+1)D T}{2}}
    \end{split}
\end{equation}
Now, we have both the upper bound and lower-bound for $\mathbb{E}_{Z}[\langle W^*,W^{T+1} \rangle]$. We compare these two bounds, to obtain
\begin{equation}
    \begin{split}
        \mathbb{E}_{Z}[{M}] - \sqrt{\frac{\lambda_1D \mathbb{E}_{Z}[{M}]}{2}} - \left(R_{T}+\sqrt{\frac{(\lambda_2+1)D T}{2}} \right) \leq 0
    \end{split}
\end{equation}
Some simple algebraic manipulation as shown in Proposition~\ref{prop1} shows that
\begin{equation}
    \begin{split}
        \mathbb{E}_{Z}[{M}] & \leq R_{T} + \sqrt{\frac{\lambda_1DL}{2}} \\
        {} &+ 3\max\left(\frac{\lambda_1D}{2},\sqrt{\frac{(\lambda_2+1)D T}{2}}\right)
    \end{split}
\end{equation}
Now as $\mathbb{E}_{Z}[{M}] + \gamma T \geq \mathbb{E}_{Z}[\Hat{M}]$ because we are not exploring for more than $\gamma T$ rounds, we get
\begin{equation}
    \begin{split}
        \mathbb{E}_{Z}[\Hat{M}] & \leq R_{T} + \sqrt{\frac{\lambda_1DL}{2}} \\
        {} &+ 3\max\left(\frac{\lambda_1D}{2},\sqrt{\frac{(\lambda_2+1)D T}{2}}\right) + \gamma T
    \end{split}
\end{equation}

\begin{Proposition}
\label{prop1}
If we are given that,
\begin{equation*}
    \begin{split}
        \mathbb{E}_{Z}[\Hat{M}] - \sqrt{\frac{\lambda_1D \mathbb{E}_{Z}[\Hat{M}]}{2}} - \left(R_{T}+\sqrt{\frac{(\lambda_2+1)D T}{2}} \right) \leq 0,
    \end{split}
\end{equation*}
then we can prove that,
\begin{equation*}
    \begin{split}
        \mathbb{E}_{Z}[\Hat{M}] & \leq R_{T} + \sqrt{\frac{\lambda_1DR_{T}}{2}} + 3\max\left(\frac{\lambda_1D}{2},\sqrt{\frac{(\lambda_2+1)D T}{2}}\right)
    \end{split}
\end{equation*}
\end{Proposition}
\begin{proof}
We are given that,
\begin{equation}
    \begin{split}
        \mathbb{E}_{Z}[{M}] - \sqrt{\frac{\lambda_1D \mathbb{E}_{Z}[{M}]}{2}} - \left(R_{T}+\sqrt{\frac{(\lambda_2+1)D T}{2}} \right) \leq 0
    \end{split}
\end{equation}
As $\mathbb{E}_{Z}[{M}] \geq 0$, let us assume that $x^2 = \mathbb{E}_{Z}[{M}]$. Therefore, the equation becomes,
\begin{equation}
    \begin{split}
        x^2 - \sqrt{\frac{\lambda_1D}{2}}x - \left(R_{T}+\sqrt{\frac{(\lambda_2+1)D T}{2}} \right) \leq 0
    \end{split}
\end{equation}
This is the equation of a convex quadratic equation, with roots as:
\begin{equation}
    \begin{split}
        x = \frac{\sqrt{\frac{\lambda_1D}{2}} \pm \sqrt{\frac{\lambda_1D}{2}+4\left(R_{T}+\sqrt{\frac{(\lambda_2+1)D T}{2}}\right)}}{2}
    \end{split}
\end{equation}
As the Discriminant is positive, we have real roots and therefore the value of quadratic is negative between the roots. We need to bound $x^2 = \mathbb{E}_{Z}[{M}]$, we obtain that,
\begin{equation*}
    \begin{split}
        \mathbb{E}_{Z}[{M}] &= x^2\\
        {} & \leq \Bigg(\frac{\sqrt{\frac{\lambda_1D}{2}} + \sqrt{\frac{\lambda_1D}{2}+4\left(R_{T}+\sqrt{\frac{(\lambda_2+1)D T}{2}}\right)}}{2}\Bigg)^2\\
        {} & \leq \frac{1}{4} \Bigg[\frac{\lambda_1D}{2} + \frac{\lambda_1D}{2}+4\left(R_{T}+\sqrt{\frac{(\lambda_2+1)D T}{2}}\right) \\
        {} & + 2\sqrt{\frac{\lambda_1D}{2}} \sqrt{\frac{\lambda_1D}{2}+4\left(R_{T}+\sqrt{\frac{(\lambda_2+1)D T}{2}}\right)} \Bigg]
    \end{split}
\end{equation*}
Using the inequality $\sqrt{a+b} \leq \sqrt{a}+\sqrt{b}$ we get,
\begin{equation}
    \begin{split}
        x^2 & \leq R_{T} + \frac{\lambda_1D}{2} +\sqrt{\frac{(\lambda_2+1)D T}{2}} \\
        {} & + \sqrt{\frac{\lambda_1DR_{T}}{2}} + \sqrt{\frac{(\lambda_2+1)D T}{2}}\sqrt{\frac{\lambda_1D}{2}}\\
        {} & \leq R_{T} + \sqrt{\frac{\lambda_1DR_{T}}{2}} \\
        {} &+ 3\max\left(\frac{\lambda_1D}{2},\sqrt{\frac{(\lambda_2+1)D T}{2}}\right)
    \end{split}
\end{equation}
Therefore, we prove the upper bound on $\mathbb{E}_{Z}[{M}]$ as
\begin{equation}
    \begin{split}
        \mathbb{E}_{Z}[{M}] & \leq R_{T} + \sqrt{\frac{\lambda_1DR_{T}}{2}} \\
        {} &+ 3\max\left(\frac{\lambda_1D}{2},\sqrt{\frac{(\lambda_2+1)D T}{2}}\right)
    \end{split}
\end{equation}
\end{proof}

\section{Proof of Corollary 8}
Using $R_T=0$, we get
\begin{align*}
        \mathbb{E}_{Z}[\Hat{M}] & \leq    3\max\left(\frac{\lambda_1D}{2},\sqrt{\frac{(\lambda_2+1)D T}{2}}\right)+ \gamma T.
\end{align*}
There are 2 cases that we need to consider:
\begin{enumerate}
    \item When $(\frac{\lambda_1D}{2} \leq\sqrt{\frac{(\lambda_2+1)D T}{2}})$:
    \begin{align*}
    \mathbb{E}_{Z}[\Hat{M}] & \leq 3\left(\sqrt{\frac{(\lambda_2+1)D T}{2}}\right)+ \gamma T.
    \end{align*}
    Using the inequality $\sqrt{a+b} \leq \sqrt{a}+\sqrt{b}$ we get:
    \begin{align}
    \mathbb{E}_{Z}[\Hat{M}] & \leq 3\left(\sqrt{\frac{\lambda_2D T}{2}}+\sqrt{\frac{D T}{2}}\right)+ \gamma T.
    \end{align}
    Also, $\lambda_2$ is given by:
    \begin{equation*}
    \begin{split}
        \lambda_2 &= \big[\frac{mk^m\;\Myperm[k]{m}}{\gamma^m\tau_1^2 }-\frac{2m\tau_2}{\tau_1} -\frac{2}{\tau_1}+{k}[\tau_2^2+\frac{1}{km}+\frac{2\tau_2}{k}]\big]\\
        {} &\leq\big[\frac{mk^m\;\Myperm[k]{m}}{\gamma^m\tau_1^2 }+{k}[\tau_2^2+\frac{1}{km}+\frac{2\tau_2}{k}]\big]\\
        {} &= [\frac{mk^m\;\Myperm[k]{m}}{\gamma^m\tau_1^2 }+c_1]]
    \end{split}
    \end{equation*}
    Substituting this in Eq.~(8) and using the inequality $\sqrt{a+b} \leq \sqrt{a}+\sqrt{b}$ we get:
    \begin{align*}
    \mathbb{E}_{Z}[\Hat{M}] & \leq 3\left(\sqrt{\frac{mk^m\;\Myperm[k]{m} D T}{2\gamma^m\tau_{1}^2}}+\sqrt{\frac{mc_1T}{2}}+\sqrt{\frac{D T}{2}}\right)+ \gamma T.
    \end{align*}
    Now, we need to find optimal value of $\gamma$ which maximizes the L.H.S of the above equation. So, we differentiate it with respect to $\gamma$ to get:
    \begin{align}
        T - \frac{3}{2\sqrt{2}}\left(\frac{m^{\frac{3}{2}}k^{\frac{m}{2}}\;(\Myperm[k]{m})^{\frac{1}{2}} D^{\frac{1}{2}}}{\tau_1}\right) \frac{T^{\frac{1}{2}}}{\gamma^{\frac{m+2}{2}}} = 0
    \end{align}
    Therefore, optimal value of $\gamma = \left(\frac{9m^{{3}}k^{{m}}(\Myperm[k]{m}) D}{8\tau_1 T}\right)^{\frac{1}{m+2}}$, which is of the form $\gamma = \left(\frac{c_2}{T}\right)^{\frac{1}{(m+2)}}$.
    
    Substituting $\gamma$ in Eq.~(9) we get,
    \begin{align*}
      \mathbb{E}_{Z}[\Hat{M}] \leq \mathcal{O}(T^{(1-\frac{1}{(m+2)})})
    \end{align*}
    
    \item When $(\frac{\lambda_1D}{2} > \sqrt{\frac{(\lambda_2+1)D T}{2}})$:
    \begin{align*}
    \mathbb{E}_{Z}[\Hat{M}] & \leq 3\left(\frac{\lambda_1D}{2}\right)+ \gamma T.
    \end{align*}
    Setting value of $\gamma \leq \mathcal{O}(\frac{1}{T}^{\frac{1}{(m+2)}})$, we get:
    \begin{align*}
       \mathbb{E}_{Z}[\Hat{M}] \leq \mathcal{O}(T^{(1-\frac{1}{(m+2)})})
    \end{align*}
    Proof of Collorary~(9), where $R_{T} \leq \mathcal{O}(T^{(1-\frac{1}{(m+2)})})$, can be done in a similar manner as described above.
    
\end{enumerate}

\end{document}